\documentclass[accepted]{uai2026} 
                        
\usepackage[american]{babel}

\usepackage{natbib} 
\usepackage{mathtools} 
\usepackage{booktabs} 
\usepackage{tikz} 
\usetikzlibrary{arrows.meta}

\usepackage{amsmath,amssymb,amsthm,amsfonts}
\newtheorem{theorem}{Theorem}

\newcommand{\bqn}{\begin{eqnarray}}
\newcommand{\eqn}{\end{eqnarray}}
\newcommand{\bq}{\begin{eqnarray*}}
\newcommand{\eq}{\end{eqnarray*}}

\author[1]{Moo~K.~Chung}
\author[2]{Anass~B.~El-Yaagoubi}
\author[2]{Hernando~Ombao}

\affil[1]{%
    Department of Biostatistics and Medical Informatics\\
    University of Wisconsin\\
    Madison, Wisconsin, USA
}
\affil[2]{%
    Statistics Program\\
    King Abdullah University of Science and Technology\\
    Thuwal, Saudi Arabia
}

\begin{document}
\title{Vector Space of Cycles}

\maketitle

\begin{abstract}
Most statistical and machine learning methods for directed interactions focus on pairwise effects among variables. Even cyclic models represent feedback mainly through node-level dependencies, making large-scale recurrent organization difficult to estimate and compare. This limitation is acute in biological and neural systems, where interactions are highly recurrent and involve many overlapping cycles.

We introduce a variational framework for statistical inference on cyclic interactions. Directed interactions are represented as edge flows on a simplicial complex and evolved under an energy-minimizing dynamical system. The dynamics separate transient components from persistent harmonic flows, yielding a low-dimensional cycle space that captures stable recurrent organization. Rather than enumerating cycles, the framework represents cyclic interactions as elements of a Hilbert space, enabling projection, averaging, comparison, and population-level inference.

We establish theoretical properties of the harmonic projection, including characterization of the cycle space, variance reduction, and population inference. Simulations show improved recovery of cyclic structure in dense recurrent systems compared with existing directed-interaction methods. Applied to resting-state fMRI from 400 human subjects, the framework reveals reproducible large-scale cyclic organization not detectable through edgewise averaging. These results provide a scalable statistical framework for recurrent interactions in high-dimensional dynamical systems.
\end{abstract}

\section{Introduction}

Functional relationships between variables have traditionally been characterized through undirected measures of association, including covariance and correlation \citep{hotelling.1992,anderson.1984}, as well as more general measures of statistical dependence \citep{szekely.2007}. Statistical modeling of multivariate dependence subsequently evolved through conditional independence \citep{dawid.1979,cox.1993} and graphical models \citep{wermuth.1990,lauritzen.2002}. While these frameworks effectively characterize statistical dependence, they do not distinguish the direction of interaction or information flow.

This limitation motivated statistical methods for modeling directed dependence. Conditional independence provides a probabilistic framework for characterizing directed relationships among variables \citep{dawid.1979}. Graphical models represent multivariate directed dependence through conditional independence relations encoded by directed graphs \citep{wermuth.1990,cox.1993,lauritzen.2002}. Structural equations describe directed dependence through systems of regression equations \citep{spirtes.2000,pearl.2009,peters.2017}. Granger-causal formulations estimate pairwise directed interactions in multivariate time series \citep{granger.1969,piancastelli.2023}. Transfer entropy provides a complementary information-theoretic measure of directed dependence \citep{schreiber.2000,redondo.2025}. Bayesian network methods similarly infer directed dependence under probabilistic graphical models \citep{driver.2017}. Machine learning approaches likewise estimate directed interactions under structural constraints \citep{spirtes.2000,pearl.2009,peters.2017}.

Despite their methodological differences, these approaches share a common statistical formulation in which inference is performed on node-level parameters or pairwise directed interactions. Network-level organization is therefore obtained only by aggregating these local estimates. Although highly successful for estimating directed effects, existing frameworks do not explicitly represent cyclic interactions as inferential objects. This limitation becomes particularly important in systems whose dynamics are dominated by recurrent interactions rather than isolated pairwise effects. Recurrent organization is a defining characteristic of many biological and neural systems \citep{keilholz.2017,anand.2023.TMI}, and is similarly observed in social and economic networks \citep{forre.2018}. Although existing directed models naturally accommodate recurrent connectivity, cyclic organization remains encoded only implicitly through collections of directed edges \citep{pearl.2009,peters.2017}. As a result, higher-order cyclic organization cannot be directly estimated, compared, averaged, or subjected to population-level statistical inference.

When the primary object of interest is intrinsically nonlinear, a common strategy in modern statistics is to construct an appropriate representation on which standard inferential procedures become well defined. Functional data analysis represents random functions as elements of Hilbert spaces, enabling estimation, regression, and hypothesis testing through linear operators \citep{wang.2016.functional}. Manifold-valued data are commonly analyzed by mapping observations to tangent spaces, where Euclidean statistical methods become locally applicable \citep{dryden.2009}. 
Object-oriented data analysis develops statistical methodology for complex data objects by constructing representations that preserve their geometric structure while enabling standard statistical operations \citep{marron.2014,marron.2021}. Fr\'echet statistics establish averaging, regression, and inference for random objects in general metric spaces \citep{petersen.2019}. These developments illustrate a common statistical principle: constructing an appropriate representation that enables standard statistical inference. 

Motivated by this principle, we develop a statistical framework for directed cyclic interactions. Our approach represents cyclic interactions in a low-dimensional harmonic vector space derived from variational principles in Dirichlet--Hodge diffusion \citep{chung.2026.ISBI,schaub.2020}. This representation avoids explicit cycle enumeration while enabling projection, averaging, covariance estimation, regression, comparison, and hypothesis testing through standard linear operations. In this framework, cyclic interactions themselves, rather than individual directed edges, constitute the primary inferential object.

\paragraph{Related work.}
Classical graph-theoretic methods characterize networks through local and global summaries, including degree, clustering, modularity, path length, and centrality \citep{bassett.2017,newman.2018}. Cyclic organization is typically characterized indirectly through path redundancy, strongly connected components, or feedback motifs, or recovered explicitly through cycle enumeration algorithms \citep{tarjan.1972,johnson.1975}. Although effective for sparse or moderately sized networks, explicit cycle recovery becomes computationally prohibitive in dense recurrent systems containing numerous overlapping cycles.

Persistent homology and topological data analysis (TDA) characterize cycles through homological invariants across filtration scales \citep{lee.2014.MICCAI}. Statistical inference for persistent homology has subsequently been developed through persistence landscapes, confidence sets, and related methodologies \citep{wasserman.2018}. These methods detect the presence and persistence of cycles but do not directly model directed interaction flows or support statistical inference on directed cyclic interactions \citep{chazal.2014}. Hodge-theoretic methods analyze directed edge flows by decomposing them into gradient, curl, and harmonic components \citep{bourakna.2024,lim.2020}. Existing work has primarily focused on flow decomposition and network characterization rather than statistical inference on cyclic interactions. In contrast, we represent cyclic interactions as elements of a harmonic vector space, transforming recurrent organization into a statistical object that admits averaging, comparison, covariance estimation, and population-level inference.

Our work is also related to cyclic structural causal models \citep{richardson.1996,lacerda.2008,mooij.2013} and their probabilistic formulations \citep{forre.2018,bongers.2021}. Despite allowing recurrent connectivity, these models retain node- or edge-level parameterizations of directed dependence, with cyclic organization represented only implicitly through collections of directed interactions. Statistical inference is therefore performed on structural coefficients or conditional distributions rather than directly on recurrent cyclic interactions. In contrast, our framework treats cyclic interactions as the primary inferential object by representing them in a harmonic vector space, enabling projection, averaging, comparison, covariance estimation, and population-level inference.

Finally, variational and energy-based methods have long provided a mathematical framework for studying dynamical systems \citep{goldstein.1950,arnold.2013,ambrosio.2005}. More recently, Hamiltonian and Lagrangian neural networks have incorporated physical principles into learned dynamical representations \citep{greydanus.2019,lutter.2019,cranmer.2020}. Our objective is different. We employ an energy-minimization principle to identify and separate transient interactions from persistent recurrent organization, yielding a low-dimensional cycle space that supports averaging, comparison, covariance estimation, and population-level statistical inference.

\paragraph{Contributions.}
(i) We introduce a variational framework for identifying persistent cyclic interactions in directed networks. Under Dirichlet--Hodge diffusion, transient interaction components dissipate while harmonic flows remain as stable recurrent structures, providing a principled mechanism for separating persistent cyclic organization from transient edgewise fluctuations. (ii) We formulate cyclic interactions as elements of a vector space. Rather than representing cycles as combinatorial objects requiring explicit enumeration, we represent them as harmonic flows in a low-dimensional Hilbert space. This transforms cycle analysis into linear operations such as projection, averaging, and comparison. (iii) We establish statistical properties of the cycle-space representation. We characterize the harmonic subspace, derive variance-reduction properties of harmonic projection, and develop population-level inference based on directional statistics on the harmonic unit sphere. (iv) We validate the proposed framework through both synthetic experiments and neuroimaging applications. Simulations demonstrate substantially improved recovery of cyclic interactions in dense recurrent networks. We further apply the proposed methodology to resting-state functional magnetic resonance imaging (rs-fMRI) data from the Human Connectome Project (HCP), illustrating its ability to identify persistent recurrent organization in large-scale brain networks.

\paragraph{Scope and limitations.}
The proposed framework is not intended to replace Pearl--Rubin causal identification or to estimate intervention effects \citep{pearl.2009,peters.2017}. Its target is narrower: statistical inference on persistent cyclic interactions in feedback-dominated dynamical systems. The method assumes that directed edge flows are first estimated from the observed data and then used as input to the proposed energy-minimization principle. Thus, its performance depends on the quality of the initial directed interaction estimates. When the underlying system contains little recurrent structure, or when there is no cycle, the harmonic subspace is trivial and no persistent cyclic organization can be inferred.

\section{Lagrangian Dynamics}

The classical Lagrangian framework \citep{arnold.2013} describes dynamical systems through an energy-based variational principle. In conservative systems, such as an ideal pendulum, energy is exchanged between kinetic and potential forms and motion may persist indefinitely. Here, we use this variational viewpoint to model how directed interactions evolve under geometric and topological constraints.

\paragraph{Simplicial complex for encoding interactions.}
To model higher-order interaction structure, we represent  interactions as a simplicial complex consisting of vertices $V$ (0-simplices), edges $E$ (1-simplices), faces $F$ (2-simplices), and higher-dimensional simplices (Figure~\ref{fig:IBC-pipeline}) \citep{edelsbrunner.2010}. This generalizes pairwise graphs by encoding multi-way structure as higher-dimensional simplices, providing a combinatorial scaffold for topological and variational analysis. Given multivariate observations $Z(t)$ at the vertices, we construct a time-varying directed edge flow
\(
X(t)\in\mathbb{R}^{|E|},
\)
where each component quantifies directed interaction along an oriented edge. The edge flow may be instantiated using any directed interaction measure, including time-lagged dependence, coherence, Granger-type residuals, or transfer entropy \citep{keilholz.2017,barnett.2014,peters.2017,schreiber.2000}. Thus, the framework is estimator-agnostic and treats $X(t)$ as the input interaction field for subsequent cycle-space analysis.

The topology of the simplicial complex is encoded by boundary matrices $\mathbf{B}_1\in\mathbb{R}^{|V|\times|E|}$ and $\mathbf{B}_2\in\mathbb{R}^{|E|\times|F|}$  \citep{anand.2023.TMI,lim.2020}. Acting on an edge flow, $\mathbf{B}_1$ measures node-level source--sink imbalance, while $\mathbf{B}_2^\top$ measures circulation around filled faces. These operators define the discrete differential structure used to quantify transient imbalance, local circulation, and persistent global cyclic organization within a unified variational framework.

\begin{figure}[t]
	\centering
	\includegraphics[width=1\linewidth]{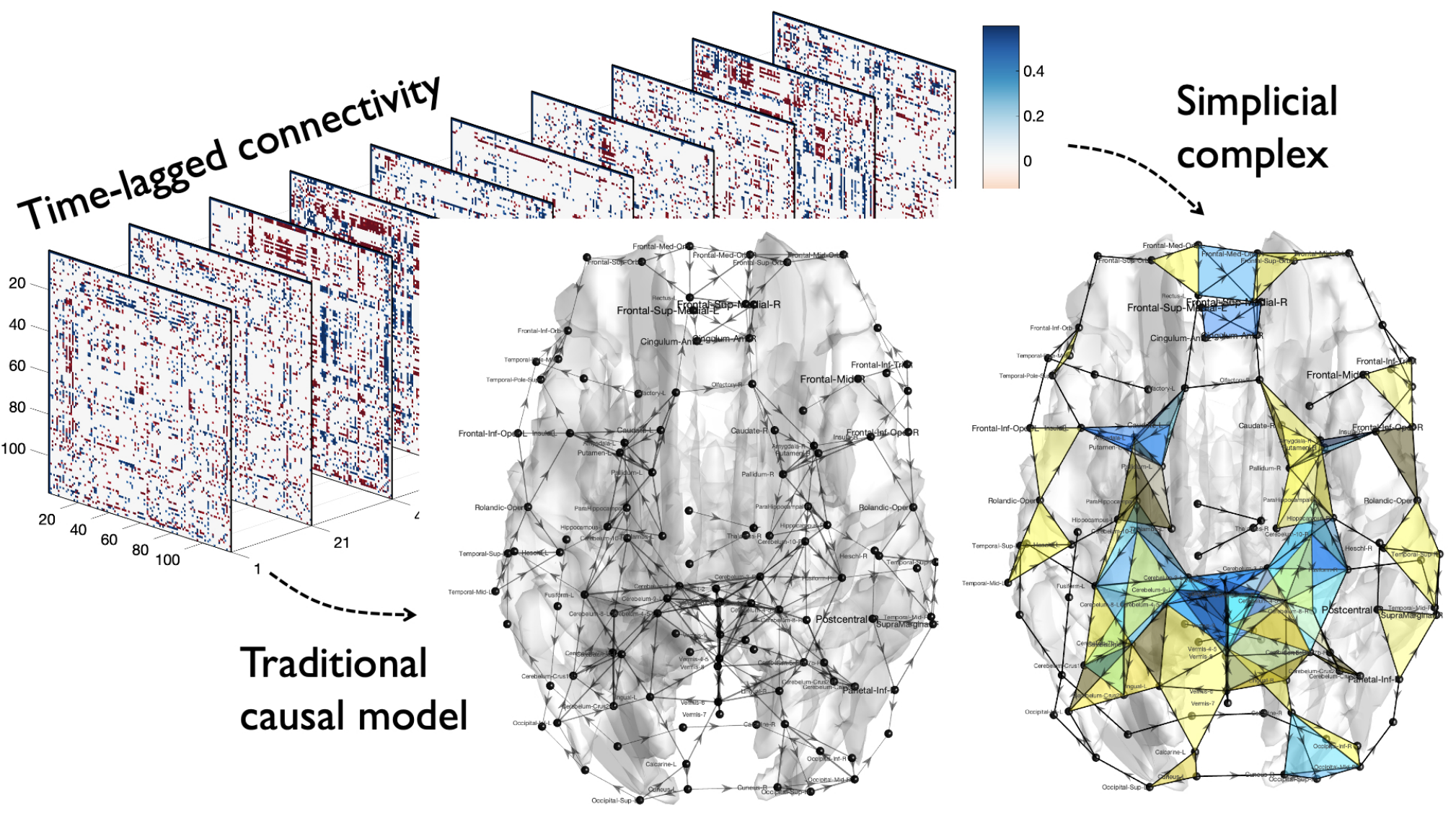}
	\caption{
{\bf Left:} Time-lagged dynamic connectivity matrix representing directed interaction dynamics.
{\bf Middle:} Conventional directed-network analysis represents pairwise interactions on a graph.
{\bf Right:} The proposed framework represents interactions on a simplicial complex encoding nodes, edges, and faces. The topology does not impose cyclic organization \emph{a priori}; persistent cyclic structure is inferred from the observed interaction flow through variational dynamics on this scaffold.
}
\label{fig:IBC-pipeline}
\end{figure}

\paragraph{Lagrangian.}
We represent directed interactions on a network by a time-dependent \emph{edge flow}
\(X(t)\in\mathbb{R}^{|E|}\), where each component of \(X(t)\) encodes the directed strength of interaction along an edge at time \(t\).

Let
\(
\dot X(t)=\frac{d}{dt}X(t)
\)
denote its time derivative. We define the Lagrangian governing the evolution of the interaction flow as
\[
\mathcal{L}(X,\dot X)
=
\frac12 \|\dot X\|_2^2
-
\mathcal{E}(X).
\]
The first term is a kinetic energy that penalizes rapid temporal changes in the interaction flow, promoting temporal smoothness. The second term is a {\it Dirichlet potential energy} \citep{chung.1997} induced by the topology of the simplicial complex,
\(
\mathcal{E}(X)=\frac12 X^\top \Delta_1 X,
\)
where \(\Delta_1\) is the \(1\)-Hodge Laplacian  \citep{anand.2023.TMI,lim.2020,schaub.2020},
\(
\Delta_1
=
\mathbf{B}_1^\top \mathbf{B}_1
+
\mathbf{B}_2\mathbf{B}_2^\top.
\)
Together, the Lagrangian balances smooth temporal evolution with global topological coherence of the directed interaction flow.

\paragraph{Action.}
Given the Lagrangian \(\mathcal{L}(X,\dot X)\), the \emph{action} of an edge-flow trajectory \(X(t)\) is defined as \citep{arnold.2013}
\(
\mathcal{A}[X]
=
\int_0^\infty \mathcal{L}(X(t),\dot X(t))\,dt.
\)
The action assigns a scalar cost to the full temporal evolution of the interaction flow, balancing dynamical smoothness with topological coherence. In conservative systems, admissible trajectories are stationary points of the action. Using variational notation \(\delta\) \citep{arnold.2013}, this condition is written as
\(
\delta \mathcal{A}[X] = 0.
\)
The resulting stationary trajectories satisfy the Euler--Lagrange equation governing the temporal evolution of the interaction flow. For the second time derivative
\( \ddot X(t)=\frac{d^2}{dt^2}X(t) \), we have:

\begin{theorem}[Euler--Lagrange dynamics of interaction flows]
\label{thm:ELD}
An edge-flow trajectory \(X(t)\) is a stationary point of the action functional \(\mathcal{A}[X]\) if and only if it satisfies
\(
\ddot X(t) + \Delta_1 X(t) = 0.
\)
\end{theorem}

The Euler--Lagrange equation is classical in mechanical systems; here we extended to directed interaction flows on a simplicial complex. The proof is given in Appendix. The equation describes coupled harmonic oscillators \citep{arnold.2013} on edges, with coupling determined by the topology encoded in the \(1\)-Hodge Laplacian \(\Delta_1\). Because the dynamics are conservative, the resulting trajectories are generally oscillatory and need not converge to stable configurations. Thus, conservative variational dynamics alone are insufficient for isolating persistent cyclic organization.

\paragraph{Dissipative dynamics.}
To obtain stable interaction patterns, we extend the conservative Lagrangian framework by incorporating the Rayleigh dissipation functional \citep{goldstein.1950}
\(
\mathcal{R}(\dot X)
=
\frac{\gamma}{2}\|\dot X\|_2^2,
\)
which penalizes rapid temporal variation of edge flows. The associated dissipative force is
\(
F
=
-\frac{\partial \mathcal{R}}{\partial \dot X}
=
-\gamma \dot X,
\)
so larger temporal velocities experience stronger damping. The evolution is governed by the Lagrange--d'Alembert principle \citep{arnold.2013,goldstein.1950},
\bqn
\delta \mathcal{A}[X]
=
-\int_0^\infty \langle F,\delta X\rangle dt,
\label{eq:LdAmain}
\eqn
which balances conservative variational structure with dissipation.

When \(\gamma=0\), this principle reduces to the Euler--Lagrange equation. For \(\gamma>0\), dissipation suppresses rapidly fluctuating interaction components and favors slowly varying, structurally consistent trajectories. Thus, dissipation acts as an intrinsic regularizer on the space of edge-flow trajectories.

\begin{theorem}[Dissipative dynamics of interaction flows]
\label{thm:damped_EL_edgeflow}
An edge-flow trajectory \(X(t)\) satisfies the Lagrange--d'Alembert principle if and only if it satisfies the damped Euler--Lagrange equation
\bqn
\ddot X+\gamma \dot X+\Delta_1 X =0,
\label{eq:full}
\eqn
where \(\gamma>0\) controls the rate at which transient interaction components are dissipated.
\end{theorem}

The proof is given in Appendix. The damped Euler--Lagrange equation defines coupled damped harmonic oscillators on edges, with coupling determined by the \(1\)-Hodge Laplacian \(\Delta_1\). This edge-based formulation describes how directed interactions propagate and stabilize through cycles. In contrast, using the \(0\)-Hodge Laplacian \(\Delta_0\), the standard graph Laplacian, yields coupled damped oscillators on nodes. Node-based oscillator models are widely used for synchronization and consensus dynamics \citep{olfati.2007,arenas.2008,mesbahi.2010}, but they describe state alignment rather than edge-flow circulation and therefore do not represent persistent cyclic interaction patterns.

\paragraph{Emergence of stable cyclic interactions.}

The damped Euler--Lagrange dynamics~\eqref{eq:full} exhibit stable long-time behavior. For \(\gamma\gg 1\), damping suppresses oscillatory transients and the flow relaxes toward a topology-constrained configuration determined by \(\Delta_1\). In neuronal systems, such stabilization is analogous to inhibitory mechanisms that regulate recurrent excitation and prevent runaway dynamics \citep{dayan.2005,van.1996}. Consequently, persistent cyclic organization emerges as a stable energy-consistent flow rather than a collection of transient interactions.

To characterize the asymptotic regime, introduce the slow time \(\tau=t/\gamma\), so that
\[
\frac{d}{dt}
=
\frac{1}{\gamma}\frac{d}{d\tau},
\qquad
\frac{d^2}{dt^2}
=
\frac{1}{\gamma^2}\frac{d^2}{d\tau^2}.
\]
Equation~\eqref{eq:full} becomes
\[
\frac{1}{\gamma^2}\frac{d^2X}{d\tau^2}
+
\frac{dX}{d\tau}
+
\Delta_1X
=
0.
\]
In the strongly damped regime, the inertial term is \(\mathcal{O}(\gamma^{-2})\) and negligible, yielding the first-order approximation
\(
\dot X(t)
=
-\frac{1}{\gamma}\Delta_1X(t).
\)
After rescaling time, this reduces to
\bqn
\dot X(t)
=
-\Delta_1X(t),
\label{eq:diffusionDH}
\eqn
which we refer to as the \emph{Dirichlet--Hodge diffusion}.

The Dirichlet--Hodge diffusion defines a purely dissipative evolution that monotonically decreases the Dirichlet energy and relaxes the system toward a stable configuration. Unlike conservative dynamics, which may sustain oscillations, the diffusive limit suppresses inertial effects and guarantees convergence to the harmonic component \citep{anand.2024.ISBI}. The initial interaction flow is progressively regularized by the topology encoded in \(\Delta_1\), so that transient components dissipate while only globally consistent cyclic organization persists in the long-time limit.

\paragraph{Harmonic flows.}

In the strongly dissipative regime, inertial effects vanish and the evolution is governed solely by the Dirichlet energy
\(
\mathcal{E}(X)=\frac12 X^\top \Delta_1 X.
\)
The system therefore evolves by steepest descent of \(\mathcal{E}\) in the edge-flow space. By Onsager's variational principle \citep{noirhomme.2024}, the overdamped limit of the dissipative Lagrangian dynamics yields precisely the Dirichlet--Hodge diffusion introduced above. Along this evolution, the Dirichlet energy serves as a Lyapunov function \citep{ambrosio.2005}. Differentiating along trajectories gives
\(
\frac{d}{dt}\mathcal{E}(X(t))
=
-\|\Delta_1X(t)\|_2^2
\le 0,
\)
so the energy decreases monotonically in time. Consequently, all components with positive Dirichlet energy decay, and the trajectory converges to the null space
\(
\ker(\Delta_1)
=
\{X\in\mathbb{R}^{|E|}:\Delta_1X=0\}.
\)
The limiting configuration is therefore the projection of the initial flow onto \(\ker(\Delta_1)\), consisting entirely of non-dissipative modes. These persistent modes define the \emph{harmonic flow}, which emerges intrinsically as the stable cyclic organization of the interaction flow (Figure~\ref{fig:decom}).

\begin{figure}[t]
	\centering
	\includegraphics[width=1\linewidth]{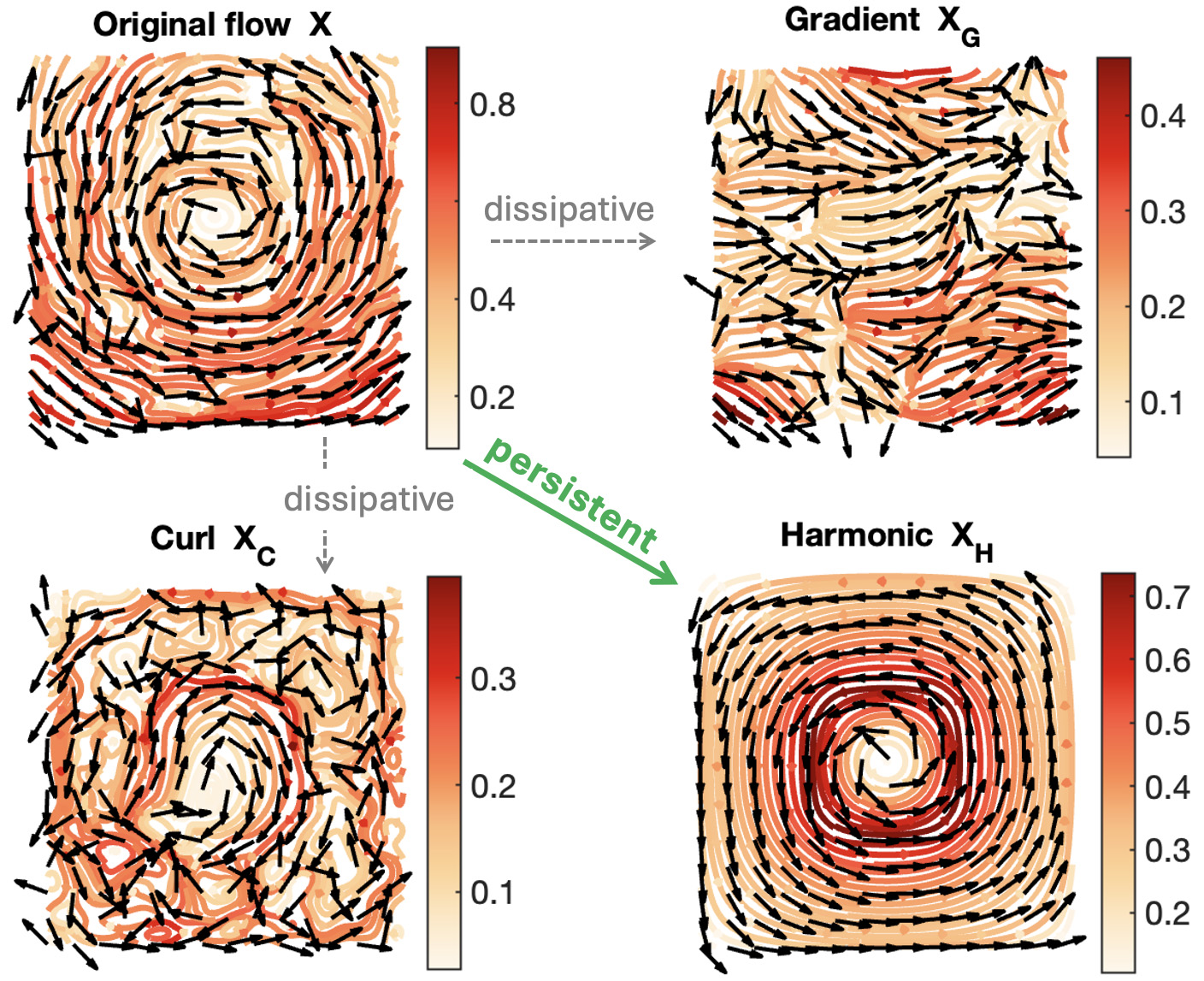}	
\caption{ 
The gradient and curl flows $X_G$ and $X_C$ correspond to dissipative structures of the original flow $X$ that decay under Dirichlet--Hodge diffusion. In contrast, the harmonic flow $X_H$ is non-dissipative and therefore persists under diffusion, encoding the globally consistent, topologically constrained circulation that remains after transient components dissipate.}
\label{fig:decom}
\end{figure}

\section{Vector Space of Cycles}

The central object of interest is the harmonic subspace $\ker(\Delta_1)$. We show that the long-time behavior of the Dirichlet--Hodge diffusion naturally leads to a vector-space representation of cyclic interactions in this space. 

\paragraph{Limit of Dirichlet--Hodge diffusion.}

Let \(\{\boldsymbol{\phi}_k\}\) be an orthonormal eigenbasis of the \(1\)-Hodge Laplacian \(\Delta_1\) with eigenvalues \(\lambda_k\ge 0\), satisfying
\(
\Delta_1 \boldsymbol{\phi}_k
=
\lambda_k \boldsymbol{\phi}_k .
\)
Expanding the initial interaction flow as
\[
X(0)
=
\sum_k \alpha_k \boldsymbol{\phi}_k,
\qquad
\alpha_k
=
\boldsymbol{\phi}_k^\top X(0),
\]
the solution of the Dirichlet--Hodge diffusion~\eqref{eq:diffusionDH} is
\[
X(t)
=
e^{-\Delta_1 t}X(0)
=
\sum_k
e^{-\lambda_k t}
\alpha_k \boldsymbol{\phi}_k .
\]
As \(t\to\infty\), all modes with \(\lambda_k>0\) decay exponentially, while modes with \(\lambda_k=0\) persist. The limiting flow is therefore
\(
X_H
=
\lim_{t\to\infty}X(t)
=
\sum_{\lambda_k=0}
\alpha_k\boldsymbol{\phi}_k.
\)
Thus, the long-time dynamics automatically filter out dissipative components and retain only the non-dissipative harmonic component.

The limiting harmonic flow can be written as
\(
X_H=\mathcal P_H X(0),
\)
where
\(
\mathcal P_H
=
\sum_{\lambda_k=0}
\boldsymbol{\phi}_k\boldsymbol{\phi}_k^\top
\)
is the orthogonal projection onto the null space
\(
\ker(\Delta_1)
=
\{X\in\mathbb R^{|E|}:\Delta_1X=0\}.
\)
The orthogonal operator \(\mathcal P_H\) extracts the persistent component of the interaction flow selected by the variational dynamics.

\paragraph{Harmonic subspace as a cycle space.}

The harmonic subspace \(\ker(\Delta_1)\) is a linear subspace of the edge-flow space. Hence, if \(X_1,X_2\in\ker(\Delta_1)\) are harmonic flows and \(a,b\in\mathbb{R}\), then
\(
\Delta_1(aX_1+bX_2)
=
a\Delta_1X_1+b\Delta_1X_2
=
0.
\)
Thus, any linear combination of harmonic flows is again harmonic. In particular, adding two admissible cyclic flows produces another admissible cyclic flow. Thus, averaging of cycles are well-defined operations that preserve cyclic consistency (Figure \ref{fig:addition}).

\begin{figure}[t]
	\centering
	\includegraphics[width=1\linewidth]{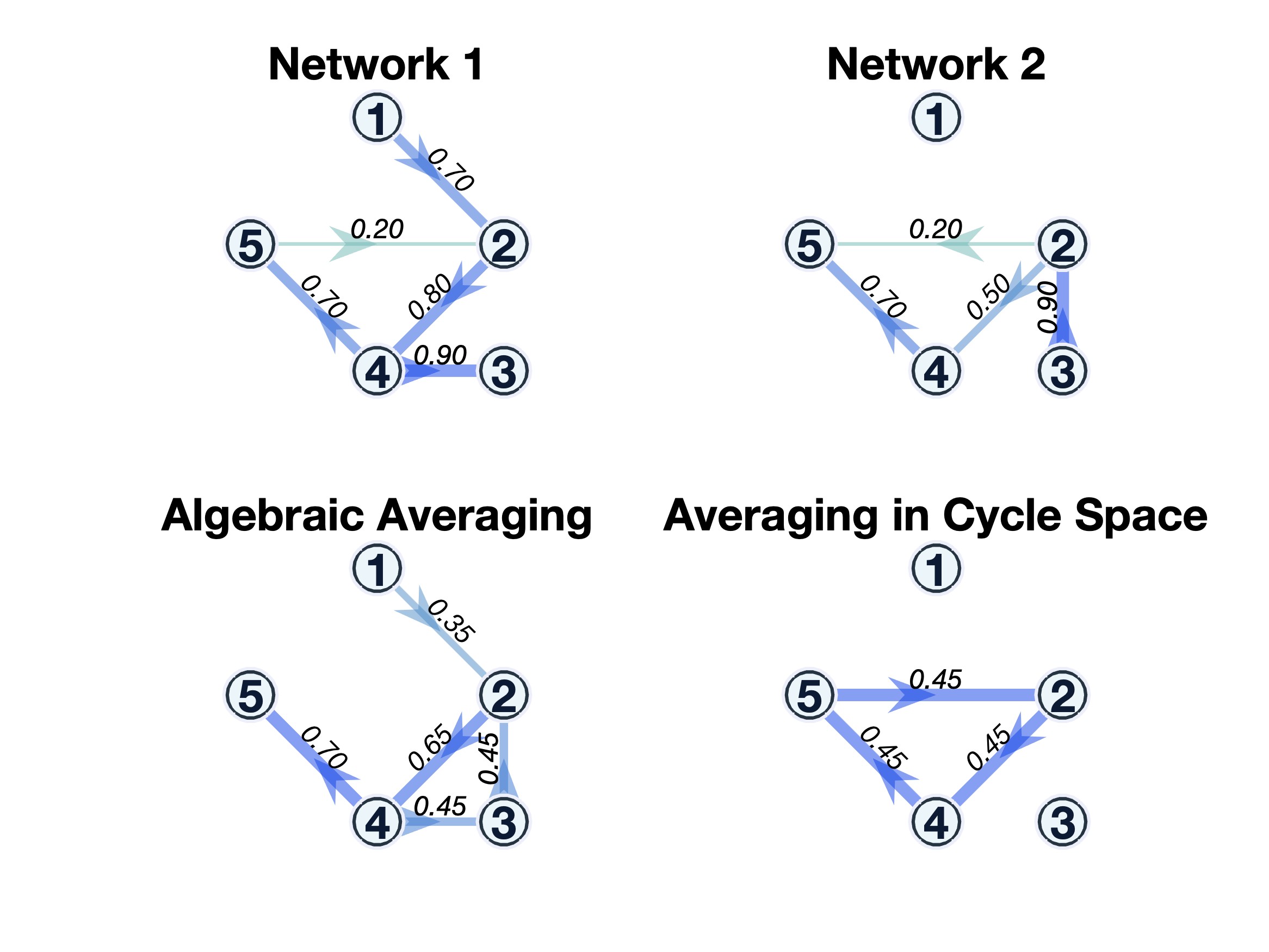}	
\caption{Two networks contain partially inconsistent cyclic structure involving \(2\to4\to5\). Direct algebraic averaging removes this cycle and instead introduces a spurious cycle \(2\to4\to3\) that is not present in either network. In contrast, averaging in the harmonic cycle space enhances the coherent recurrent cycle \(5\to2\to4\to5\).}
\label{fig:addition}
\end{figure}

In contrast, if \(X_H\in\ker(\Delta_1)\) is harmonic and \(X_D\notin\ker(\Delta_1)\) is non-harmonic, then \(X_H+X_D\) is generally non-harmonic because
\(
\Delta_1(X_H+X_D)
=
\Delta_1X_D
\neq 0.
\)
In the vector space  \(\ker(\Delta_1)\), we have well defined inner product between cycles:
$
\langle X_H,Y_H\rangle
=
X_H^\top Y_H,
$
which measures cyclic alignment and similarity between recurrent flow patterns.

Under Dirichlet--Hodge diffusion, this non-harmonic component dissipates, while the harmonic component remains invariant. Equivalently,
\(
\mathcal{P}_H(X_H+X_D)=X_H,
\)
provided \(X_D\) is orthogonal to \(\ker(\Delta_1)\). Thus, harmonic projection removes transient non-cyclic components and retains the admissible cyclic component.

\begin{theorem}[Harmonic flows are cycles]
\label{thm:harmonic_are_cycles}
Every nonzero harmonic flow \(X\in\ker(\Delta_1)\) represents a circulation around cycles of the simplicial complex.
\end{theorem}

The proof is given in Appendix. Consequently, the harmonic subspace provides a vector-space representation of cyclic interactions: cycles can be added, scaled, averaged, and projected while remaining inside the same admissible cycle space. The dimension of the harmonic subspace is
\(
\dim\ker(\Delta_1)=\beta_1,
\)
where \(\beta_1\) is the first Betti number of the simplicial complex \citep{edelsbrunner.2010}. Therefore, \(\ker(\Delta_1)\) is spanned by \(\beta_1\) independent cycle representatives, and every harmonic flow is a linear combination of these cycles.

\paragraph{Numerical implementation.}
\label{app:numerical}

The boundary matrices $\mathbf{B}_1$ and $\mathbf{B}_2$ are highly sparse, with at most two nonzero entries per column in $\mathbf{B}_1$ and three per column in $\mathbf{B}_2$. The harmonic projection does not require explicit eigendecomposition of the Hodge Laplacian, which can be computationally expensive for large complexes.
 Instead, it is computed via a constrained least-squares formulation, reducing the problem to solving a sparse linear system.
 
Given an edge flow \( X \in \mathcal{C}^1 \), we aim to extract its harmonic component \( X_H \in \ker(\mathbf{B}_1) \cap \ker(\mathbf{B}_2^\top) \), which lies in the kernel of the Hodge 1-Laplacian. Thus, the harmonic component is the unique solution to the orthogonal projection problem:
\[
X_H = \arg\min_{Z \in \mathbb{R}^{|E|}} \| X - Z \|^2  \text{ under }  \mathbf{B}_1 Z = 0,  \mathbf{B}_2^\top Z = 0.
\]
This is a least-squares estimation (LSE) problem with linear equality constraints, projecting \( X \) onto the harmonic subspace.  Let \( \mathbf{B} \) denote the stacked boundary and coboundary matrix:
$
\mathbf{B}
=
\left[\mathbf{B}_1^\top\ \mathbf{B}_2\right]^\top.
$
The orthogonal projection onto \( \ker(\mathbf{B}) \) is given by
$
\mathcal{P}_H = I - \mathbf{B}^\top (\mathbf{B} \mathbf{B}^\top)^\dagger \mathbf{B},
$
where \( (\cdot)^\dagger \) denotes the Moore–Penrose pseudoinverse. The harmonic component is then computed as
$
X_H = \mathcal{P}_H X.
$

\paragraph{Variance reduction properties.} 
\label{app:variance_isotropic} Harmonic projection improves statistical stability by removing disspiant components and isolating globally persistent cyclic structure. As a result, statistical inference is conducted in a variance-reduced subspace. To quantify this effect, let $X \in \mathbb{R}^{|E|}$ be a random edge flow with covariance $\Sigma_X$. We define the total variance
\[
\operatorname{Var}(X)
=
\mathbb{E}\|X-\mathbb{E}X\|_2^2
=
\operatorname{tr}(\Sigma_X).
\]

\begin{theorem}[Variance reduction]
If $\Sigma_X = \sigma^2 I$, then for harmonic flow $X_H = \mathcal{P}_H X$,
\[
\operatorname{Var}(X_H)
=
\frac{\beta_1}{|E|}
\operatorname{Var}(X),
\]
where $\beta_1 = \dim \ker(\Delta_1)$ is the first Betti number.
\end{theorem}

The proof is given in Appendix. Since typically $\beta_1 \ll |E|$, harmonic projection contracts noise by a factor $\beta_1/|E|$, concentrating variability into a low-dimensional topologically constrained subspace. In our human neural network application, this ratio is approximately $534/6670 \approx 0.08$, explaining the empirical stability of the population-level harmonic flow and enabling reliable group-level inference.

\paragraph{Inference in cycle space.}

For subject \(s\), let
\[
X_H^{(s)}
=
\mathcal{P}_H X^{(s)}
\in
\ker(\Delta_1)
\]
denote the subject-level harmonic flow. The population mean is
$
\mu=\mathbb{E}[X_H^{(s)}],
$
and we test the null hypothesis
$
H_0:\ \mu=0,
$
which asserts the absence of systematic cyclic organization across subjects. Rejection of \(H_0\) indicates population-level alignment of harmonic flows beyond random fluctuation.

Under \(H_0\), subjects may individually exhibit nonzero harmonic flows, but their directions fluctuate independently and do not align coherently. We therefore focus on directional alignment rather than magnitude. Define the normalized harmonic flow
\[
U^{(s)}
=
\frac{X_H^{(s)}}{\|X_H^{(s)}\|_2}
\in
\mathbb{S}_H^{\beta_1-1},
\]
which lies on the harmonic unit sphere
\[
\mathbb{S}_H^{\beta_1-1}
=
\Bigl\{
U\in\ker(\Delta_1):
\|U\|_2=1
\Bigr\}.
\]
Population-level cyclic coherence corresponds to concentration of the directions \(U^{(s)}\) around a common orientation in the harmonic cycle space (Figure \ref{fig:hyperspherebeta}).

Alignment across subjects is quantified by the geodesic distance induced by cosine similarity \citep{luo.2018,zhao.2025}:
\[
\theta_{st}
=
\cos^{-1}
\!\bigl(
\langle U^{(s)},U^{(t)}\rangle
\bigr),
\]
which measures the angular separation between subjects \(s\) and \(t\). Because cosine similarity depends only on orientation, it is invariant to scale and measures directional agreement independently of magnitude.

The sample mean direction is
\[
\bar U
=
\frac{1}{n}
\sum_{s=1}^{n}
U^{(s)}.
\]
Its norm \(\|\bar U\|_2\) measures directional concentration across subjects: random orientations cancel so that \(\bar U\) is close to zero, whereas coherent cyclic organization yields larger \(\|\bar U\|_2\).

\begin{figure}[t]
\centering
\includegraphics[width=0.7\linewidth]{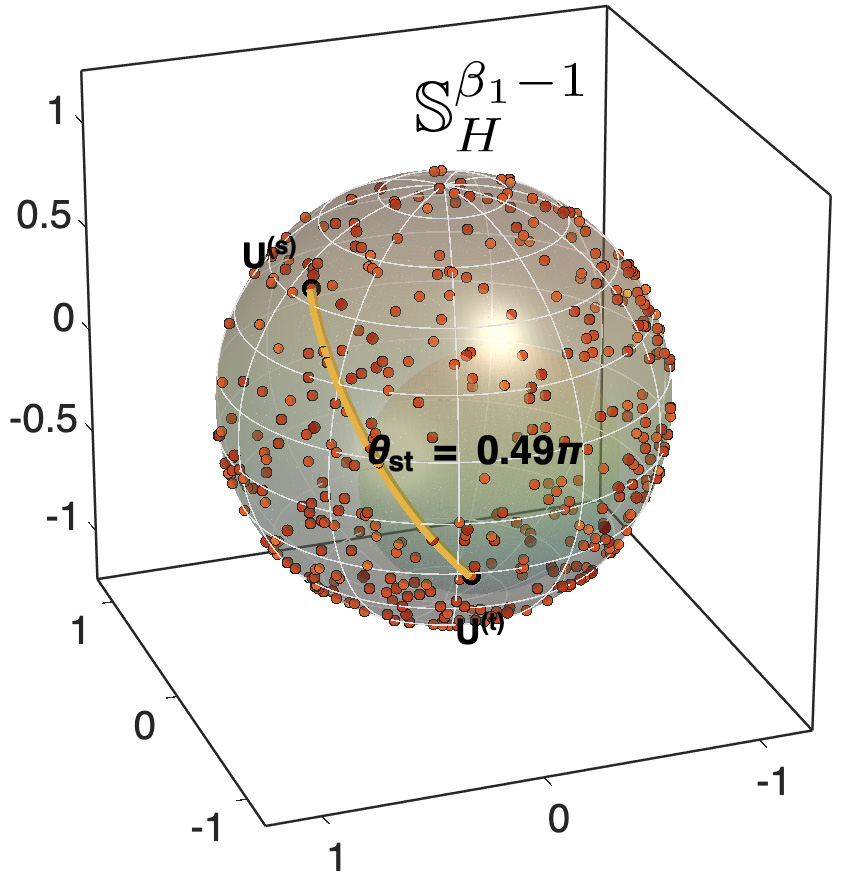}
\caption{
Normalized harmonic flows from 400 subjects lie on \(\mathbb{S}_H^{\beta_1-1}\) and are embedded onto \(S^2\) using spherical MDS for visualization. The arc shows the geodesic angle \(\theta_{st}\), where
\(
\cos \theta_{st}
=
\langle U^{(s)},U^{(t)}\rangle
\).
}
\label{fig:hyperspherebeta}
\end{figure}

Inference is performed using the Rayleigh test for mean direction \citep{mardia.2009}. Since the harmonic flows are restricted to the \(\beta_1\)-dimensional cycle space \(\ker(\Delta_1)\), the relevant dimension is \(\beta_1\), not the ambient edge dimension \(|E|\). Under \(H_0\), the normalized directions are uniformly distributed on \(\mathbb{S}_H^{\beta_1-1}\). By rotational symmetry,
\[
\mathbb{E}[U^{(s)}]
=
0,
\qquad
\mathrm{Cov}(U^{(s)})
=
\frac{1}{\beta_1}
I_{\beta_1}.
\]

By the multivariate central limit theorem \citep{casella.2024},
$
\sqrt{n}\,\bar U
\ \xrightarrow{d}\
N\!\left(
0,
\frac{1}{\beta_1}I_{\beta_1}
\right).
$
Therefore,
$
\beta_1 n \|\bar U\|_2^2
\ \xrightarrow{d}\
\chi^2_{\beta_1}.
$
The corresponding \(p\)-value is
$
p
=
\mathbb{P}
\!\left(
\chi^2_{\beta_1}
\ge
\beta_1 n \|\bar U\|_2^2
\right).
$
A significant result indicates reproducible cyclic organization across subjects in the harmonic cycle space.

\section{Validation}
\label{sec:com}

Our validation focuses on recovering cyclic interactions, the inferential target of this paper. Accordingly, the synthetic studies evaluate recovery of cyclic ground truth. We compare the proposed framework against six representative directed-interaction baselines spanning temporal prediction, instantaneous structural modeling, information-theoretic dependence, lagged correlation, DAG-constrained optimization, and cyclic constraint-based discovery: Granger causality \citep{granger.1969}, SEM \citep{driver.2017}, Bayesian MAP SVAR \citep{lutkepohl.2005,kilian.2017}, transfer entropy \citep{schreiber.2000}, lagged Pearson correlation \citep{smith.2011,keilholz.2017}, NOTEARS \citep{zheng.2018}, and CCD \citep{richardson.1996}. 


\paragraph{Ground Truth.}
We generate cyclic ground truth using random cubical complexes \citep{chungYM.2024}, which provide directed edge flows with known recurrent structure for evaluating cyclic recovery (Fig.~\ref{fig:simulation-circulation}, left). We construct a directed \(1\)-skeleton on a \(5\times5\) cubical grid, where nodes correspond to lattice vertices and edges connect horizontally and vertically adjacent vertices. Directed interactions are assigned along the boundaries of selected square faces to induce coherent circulations, producing a nontrivial harmonic subspace. Clockwise (red) and counterclockwise (blue) circulations define the embedded ground-truth cyclic interactions, providing known cycle orientations and strengths against which estimated harmonic flows can be quantitatively compared.

Time series data at each node are  simulated so that their dependencies follow the prescribed ground-truth edge flows. We then generate a vector autoregressive (VAR) process of order $L$ \citep{gong.2024.causal,papana.2017,primiceri.2005}:
\[
Z_t = \sum_{\ell=1}^L A_\ell Z_{t-\ell} + \varepsilon_t, 
\qquad 
\varepsilon_t \sim N(0, \sigma^2 I),
\]
where $Z_t \in \mathbb{R}^{|V|}$ represents node-level activity at time $t$.
The lag matrices $A_\ell$ are constructed directly from the ground-truth edge-flow matrix $W=(w_{ij})$, where $w_{ij}$ encodes the directed interaction strength from node $i$ to node $j$. For each lag,
\(
A_\ell 
=
\rho I
+
c\,\delta^{\,\ell-1} W,
\)
where $\rho = 0.1$ controls self-retention (the extent to which each node preserves its own past activity), $c = 0.1$ scales interaction strength along directed edges, and $\delta = 0.5$ imposes geometric decay across higher-order lags. For each simulation, $200$ time points are generated. A relatively long lag order ($L=10$) disperses the direct cyclic influence across time, blurring immediate feedback structure and making recovery of the underlying cyclic interactions substantially more challenging.

\begin{figure}[t]
	\centering
	\includegraphics[width=1\linewidth]{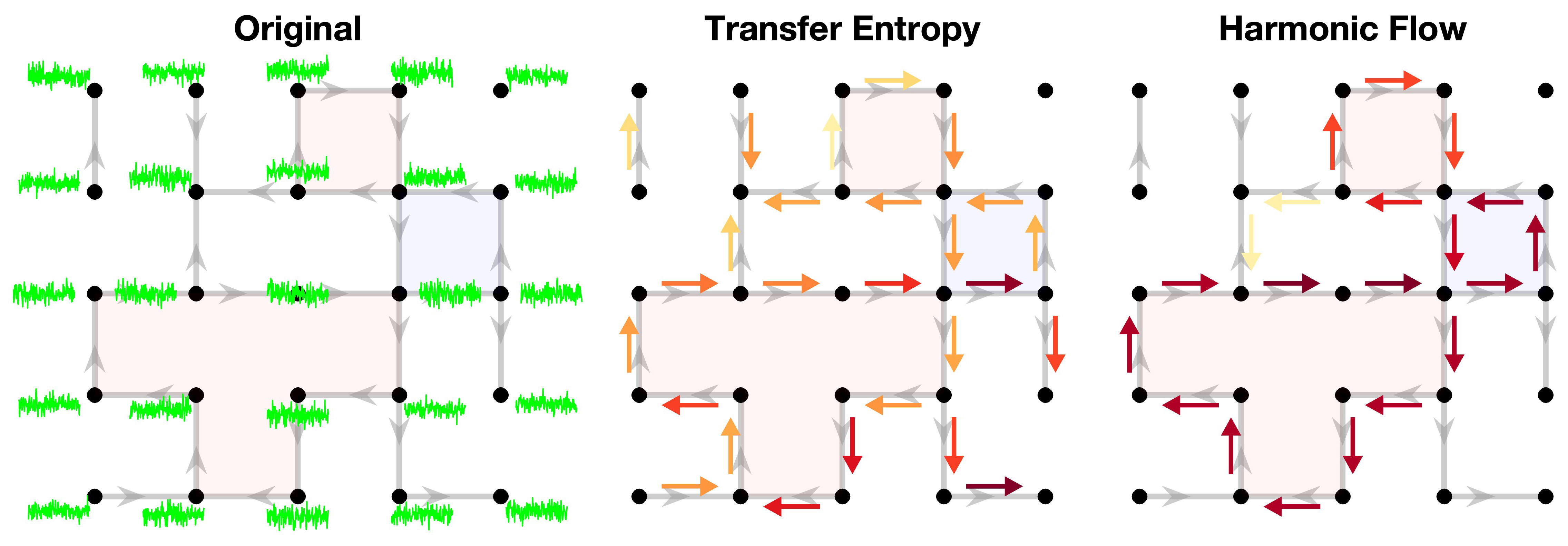}	
\caption{
{\bf Left:} Ground-truth directed edge flow on a cubical complex, consisting of two clockwise (red) and one counterclockwise (blue) circulations. Node time series are generated from a VAR-process driven by these prescribed edge flows. 
{\bf Middle:} Directed interactions estimated using transfer entropy (colored arrows). Spurious signals appear along non-cyclic edges, obscuring the underlying cyclic organization. 
{\bf Right:} Harmonic flow obtained via the Hodge projection of the transfer entropy. Color scales differ across panels because harmonic flow is the orthogonal projection of transfer entropy onto the cycle space and therefore has smaller magnitude.
\label{fig:simulation-circulation}}
	\includegraphics[width=1\linewidth]{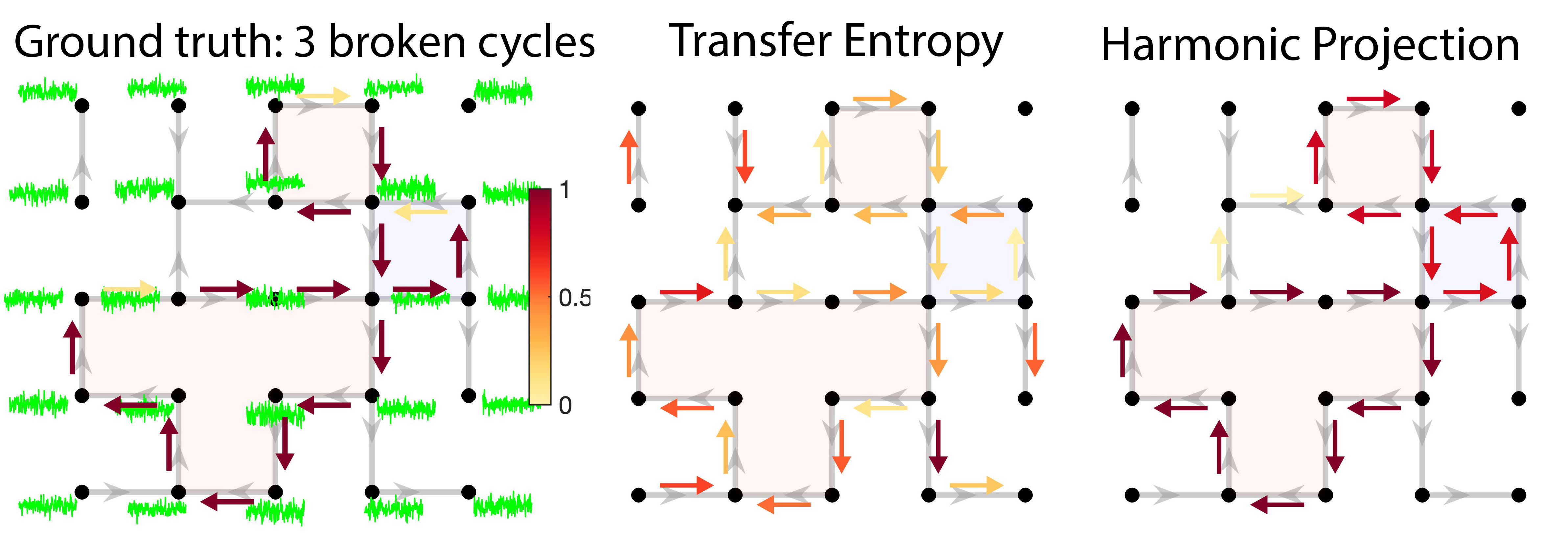}	
\caption{Simulation setting differ from Figure \ref{fig:simulation-circulation} example by 
introducing broken edges in cycles. One edge per cycle is weakened below the noise level, masking cyclic structure in the observations. Existing methods, including transfer entropy, fail to recover the cyclic organization, whereas harmonic projection recovers the underlying cycles.
\label{fig:simulation-broken}}
\end{figure}

The reported ground truth configuration is representative, and all results are based on 100 independent trials with independently generated noise (three settings $\sigma =0.1, 1, 10$) and time series. The synthetic data-generation process does {\it not} use any component of the proposed  method. All methods are blind to the ground truth cyclic information. Thus, although the evaluation target is cyclic by design, the data-generating mechanism is independent of the proposed method.

We performed stress-test experiments in two ground-truth settings: three full cycles and three broken cycles. In the broken-cycle settings, at least one edge per cycle was attenuated to $0.1$, below the noise level, thereby masking or partially disrupting cycles at the observation level. Figure~\ref{fig:simulation-broken} shows the three broken-cycle experiment in setting.

\paragraph{Results.}
Recovery accuracy was quantified using cosine similarity between the estimated and ground-truth edge flows \citep{luo.2018,zhao.2025}. Cosine similarity is scale-invariant and measures directional agreement independently of magnitude, making it particularly suitable for evaluating cyclic recovery. Unlike edgewise accuracy, it remains sensitive to recovery of the underlying circulation even when individual edges are noisy or partially disrupted.

Tables~\ref{table:full-cycle-3}--\ref{table:broken-cycle-2} summarize results over 100 simulations. In the fully observed three-cycle network, all baseline methods achieved low cosine similarities (\(0.02\)--\(0.22\)), indicating poor recovery of the true cyclic organization. After harmonic projection, recovery increased substantially across all estimators, reaching approximately \(0.79\)--\(0.83\) regardless of noise level. Similar behavior was observed when one edge in each cycle was weakened below the noise level. While baseline methods again failed (\(0.02\)--\(0.23\)), harmonic projection maintained high recovery (\(0.72\)--\(0.75\)). Performance remained remarkably stable across noise levels \(\sigma=0.1,1,\) and \(10\), demonstrating robustness to both observation noise and partially broken cycles. These results indicate that harmonic projection recovers persistent cyclic organization even when individual cycle edges are weak, noisy, or missing in the observed interactions.

\begin{table*}[t]
\centering
\caption{Cosine similarity (mean $\pm$ s.d.) over 100 trials for recovering ground truth circulations under  across different noise levels $\sigma$. The top block reports baseline methods, and the bottom block reports their harmonic projections. A cosine similarity of $1$ indicates perfect recovery of the true cyclic flow. Three interacting-cycle network with full cyclic ground truth.}
\label{table:full-cycle-3}
\footnotesize
\setlength{\tabcolsep}{1pt}
\renewcommand{\arraystretch}{1.05}
\begin{tabular}{c|c|c|c|c|c|c|c|c}
\hline
Method & Noise $\sigma$ & Granger & SEM & Bayesian & Transfer Entropy & Lagged Corr. & NOTEARS & CCD \\
\hline
Baseline 
& 0.10 & $0.19 \pm 0.02$ & $0.17 \pm 0.03$ & $0.17 \pm 0.03$ & $0.21 \pm 0.01$ & $0.21 \pm 0.05$ & $0.02 \pm 0.04$ & $0.03 \pm 0.05$ \\
& 1.00 & $0.20 \pm 0.02$ & $0.17 \pm 0.02$ & $0.17 \pm 0.02$ & $0.21 \pm 0.01$ & $0.22 \pm 0.05$ & $0.03 \pm 0.05$ & $0.02 \pm 0.04$ \\
& 10.0 & $0.19 \pm 0.02$ & $0.17 \pm 0.02$ & $0.17 \pm 0.03$ & $0.21 \pm 0.01$ & $0.21 \pm 0.05$ & $0.02 \pm 0.04$ & $0.02 \pm 0.05$ \\
\hline
Harmonic 
& 0.10 & $0.82 \pm 0.01$ & $0.80 \pm 0.02$ & $0.79 \pm 0.03$ & $0.83 \pm 0.00$ & $0.79 \pm 0.00$ & -- & -- \\
& 1.00 & $0.82 \pm 0.01$ & $0.80 \pm 0.02$ & $0.80 \pm 0.02$ & $0.83 \pm 0.00$ & $0.79 \pm 0.00$ & -- & -- \\
& 10.0 & $0.82 \pm 0.01$ & $0.79 \pm 0.03$ & $0.79 \pm 0.03$ & $0.83 \pm 0.00$ & $0.79 \pm 0.00$ & -- & -- \\
\hline
\end{tabular}
\end{table*}

\begin{table*}[t]
\centering
\caption{Three interacting-cycle simulation in which at least one edge in each cycle is broken.}
\label{table:broken-cycle-2}
\footnotesize
\setlength{\tabcolsep}{1pt}
\renewcommand{\arraystretch}{1.05}
\begin{tabular}{c|c|c|c|c|c|c|c|c}
\hline
Method & Noise $\sigma$ & Granger & SEM & Bayesian & Transfer Entropy & Lagged Corr. & NOTEARS & CCD \\
\hline
Baseline 
& 0.10 & $0.19 \pm 0.02$ & $0.16 \pm 0.03$ & $0.16 \pm 0.03$ & $0.20 \pm 0.01$ & $0.22 \pm 0.05$ & $0.02 \pm 0.04$ & $0.02 \pm 0.04$ \\
& 1.00 & $0.19 \pm 0.02$ & $0.17 \pm 0.02$ & $0.16 \pm 0.02$ & $0.20 \pm 0.01$ & $0.22 \pm 0.05$ & $0.03 \pm 0.04$ & $0.03 \pm 0.06$ \\
& 10.0 & $0.18 \pm 0.02$ & $0.16 \pm 0.03$ & $0.17 \pm 0.02$ & $0.20 \pm 0.01$ & $0.23 \pm 0.05$ & $0.02 \pm 0.04$ & $0.02 \pm 0.05$ \\
\hline
Harmonic 
& 0.10 & $0.74 \pm 0.01$ & $0.72 \pm 0.03$ & $0.72 \pm 0.03$ & $0.75 \pm 0.00$ & $0.73 \pm 0.00$ & -- & -- \\
& 1.00 & $0.74 \pm 0.01$ & $0.72 \pm 0.02$ & $0.72 \pm 0.02$ & $0.75 \pm 0.00$ & $0.73 \pm 0.00$ & -- & -- \\
& 10.0 & $0.74 \pm 0.01$ & $0.72 \pm 0.03$ & $0.72 \pm 0.03$ & $0.75 \pm 0.00$ & $0.73 \pm 0.00$ & -- & -- \\
\hline
\end{tabular}
\end{table*}

\section{Application}

\paragraph{Human Neural Networks.}
We applied the proposed framework to resting-state functional magnetic
resonance imaging (rs-fMRI) data from 400 healthy young adults (age
22–36 years, mean $29.24 \pm 3.39$ years) drawn from the Human
Connectome Project (HCP) \citep{vanessen.2012}. Resting-state fMRI measures spontaneous blood oxygenation level-dependent (BOLD) fluctuations that reflect spontaneous neuronal activity \citep{fiecas.2016,huang.2020.NM}. The resulting high-dimensional multivariate time series provide a natural application for statistical inference on directed interactions in large-scale brain networks \citep{worsley.2002,lindquist.2008,zhu.2023}.

For each subject, we analyzed 648 seconds (10.8 minutes) of data.
Directed edge flows were estimated using time-lagged Pearson
correlations computed within 20-second sliding windows advanced in
5-second increments \citep{huang.2020.NM,keilholz.2017,petri.2014}.
The window length captures lagged dependencies while preserving
temporal variation in interaction structure \citep{keilholz.2017}. The window length was chosen to cover the main temporal extent of the BOLD hemodynamic response: the response typically begins after approximately $1$--$2$ seconds, peaks around $4$--$6$ seconds, and can exhibit an undershoot over approximately $10$--$20$ seconds \citep{glover.1999}. Thus, a $20$-second window provides a practical short-window scale for estimating time-resolved rs-fMRI connectivity while retaining sufficient temporal localization. For each node pair, the dominant lagged correlation defined the
directed interaction.

\paragraph{Failure of Direct Averaging.} 
Due to inter-subject variability and temporal asynchrony, direct averaging
of time-varying edge flows largely cancels directional effects, producing
extremely weak mean connectivity (maximum correlation $=0.034$ in
Figure~\ref{fig:averageflow}, left). The Rayleigh test \citep{mardia.2009}
yields a p-value of 0.50, indicating no coherent population-level alignment.

\begin{figure}[t]
	\centering
	\includegraphics[width=1\linewidth]{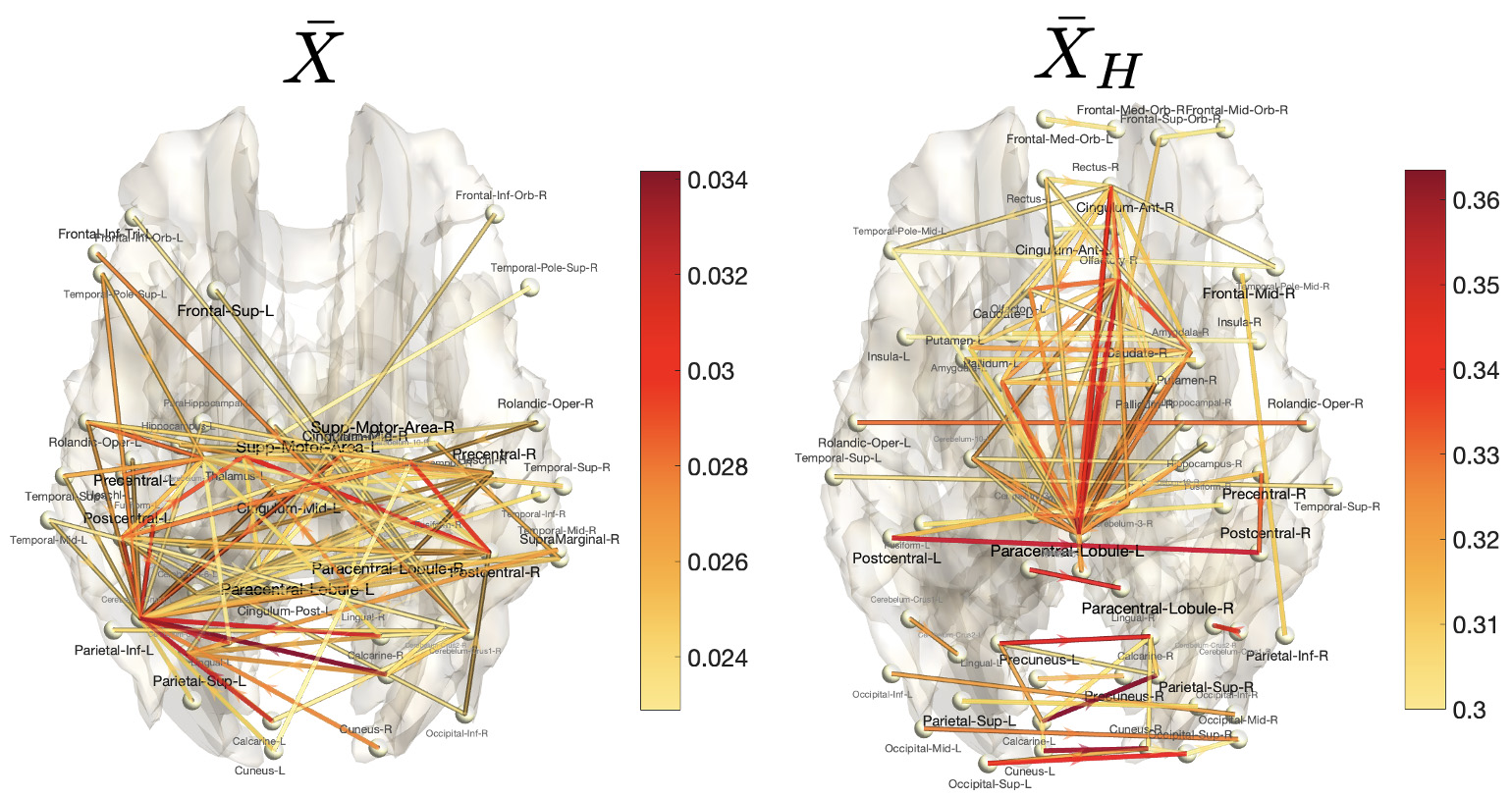}
\caption{
{\bf Left:} Average edge flow $\bar X$ over time and subjects derived from time-lagged Pearson correlations. The mean edge magnitudes are small (approximately 0.02–0.034) and statistically insignificant (p = 0.50); no edge exceeds a correlation of 0.3. Here we display the 100 largest edges, all of which remain statistically insignificant.
{\bf Right:} Average harmonic flow $\bar X_H$ over time and subjects computed from the same time-lagged correlations, which are is highly significant (p $< 10^{-29}$). The 100 largest harmonic edges (with magnitudes all above  0.3) are shown.}
\label{fig:averageflow}
\end{figure}

This reflects a fundamental challenge in resting-state network analysis:
time-resolved functional connectivity is not temporally synchronized
across subjects or across time \citep{huang.2020.NM}. 
 Transient correlations fluctuate in both sign
and magnitude, so direct averaging of dynamic directed networks is
\emph{ill-posed}; opposing directional patterns cancel, obscuring recurrent
structure. Consequently, many prior studies rely on static connectivity
summaries, sacrificing temporal and directional information
\citep{huang.2020.NM,ma.2023}.

At the same time, large-scale neural systems exhibit recurrent and
feedback-driven dynamics \citep{deco.2013,liu.2025}. Yet most brain network methods impose acyclicity or focus on feedforward dependencies
\citep{peters.2017,weichwald.2021,zheng.2018}, excluding cyclic interactions as
primary organizational principles. This mismatch limits the ability of existing models to identify persistent
recurrent structure and motivates frameworks that explicitly target cyclic interactions. The proposed harmonic flow offers a principled and
structurally grounded resolution to this longstanding challenge.

\paragraph{Results.}
Despite the absence of explicit tasks, the population-level harmonic flow
\(\bar X_H\), obtained by averaging across subjects and time, exhibits a
structured and temporally stable organization. This stability enables
meaningful aggregation of dynamic networks across individuals and time
windows (Figure~\ref{fig:averageflow}, right), in contrast to conventional
time-resolved connectivity measures, which largely cancel under averaging.
A Rayleigh test applied to the averaged harmonic flow yields
\(p < 10^{-29}\), indicating exceptionally strong population-level cyclic
alignment.

\begin{figure}[t]
	\centering
	\includegraphics[width=1\linewidth]{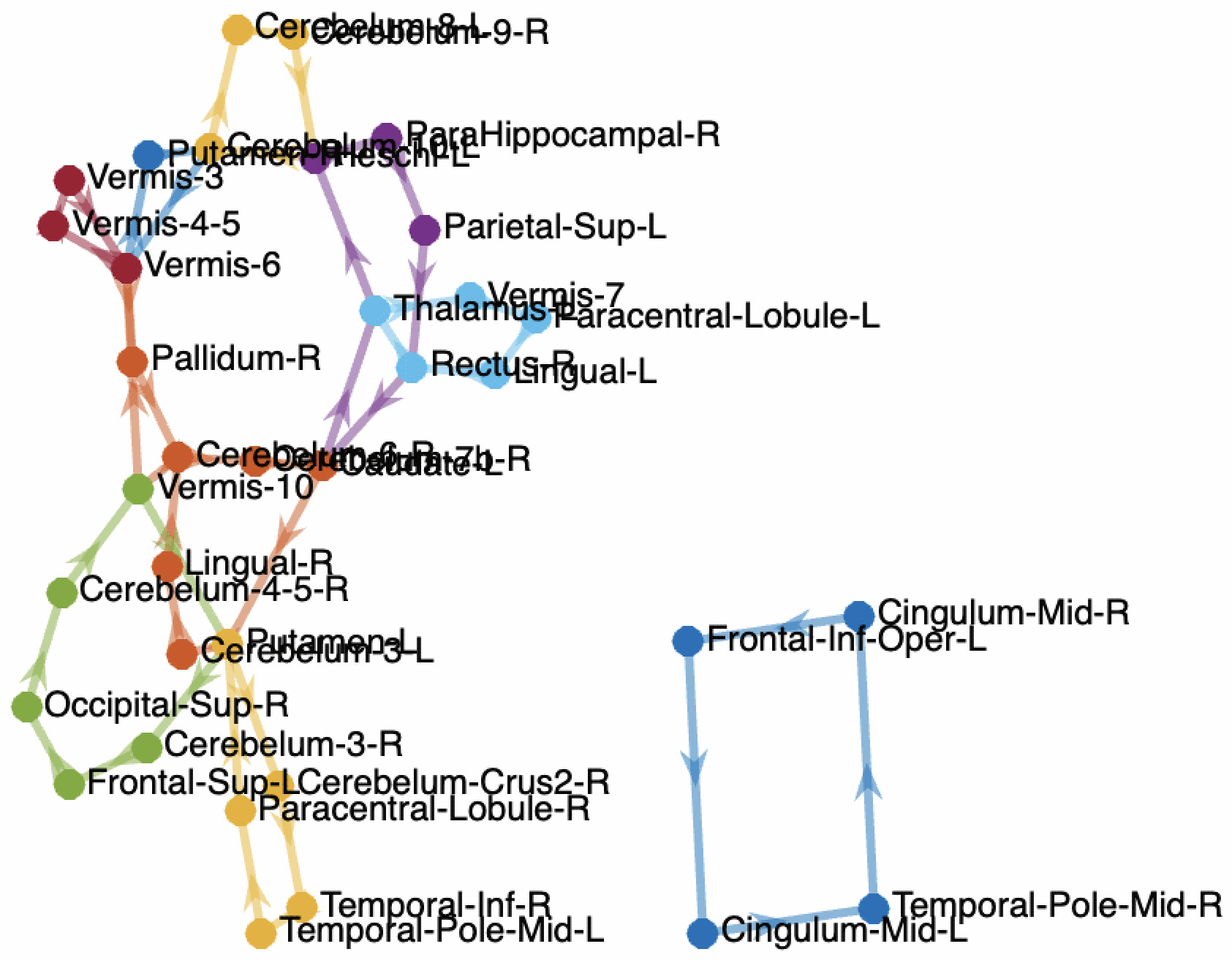}
\caption{Top 10 cycles out of 534 identified from the mean harmonic flow \(\bar U\) across 400 human brain networks ($p<10^{-29}$). The cycles repeatedly involve shared brain regions but not identical directed edges, indicating reproducible cyclic organization at the node level rather than exact edgewise recurrence. This explains why direct edge-weight averaging fails to align cycles.}
\label{fig:cyclenodes}
\end{figure}

At the group level, the harmonic component reveals a small number of
large-scale recurrent cycles that persist across subjects. These cycles
link bilaterally symmetric sensory and motor cortices with midline
structures, forming recurrent loops consistent with coordinated
interhemispheric processing. The resulting anatomically interpretable
patterns indicate that harmonic flow captures intrinsic functional
symmetries and stable feedback pathways that are difficult to recover using acyclic models or static connectivity averages.

\section{Conclusion}

We introduced a variational formulation for directed interactions based on dissipative
Lagrangian dynamics. Unlike DAG-based approaches that exclude or relax
cycles, the proposed framework treats recurrence as a primary structure of interest. In the strongly dissipative regime, the dynamics converge to a
harmonic steady state that isolates the stable cyclic component determined
by network topology. Thus cycles are not imposed structurally, but emerge
as the asymptotic limit of the dissipative dynamics.

Applied to dynamic human neural networks, raw directed edge flows largely
cancel under averaging, whereas harmonic projection reveals strong and
reproducible group-level cyclic organization. This suggests that recurrence
manifests as stable global alignment rather than as large marginal edge
effects. By selecting the harmonic component through variational dynamics,
the framework provides a principled mechanism for identifying cyclic structure as the intrinsic stable limit of dissipative dynamics in
feedback-dominated systems.

\appendix

\section{Proofs to Theorems}

\begin{proof}[Proof to Theorem 1]
Consider a perturbed edge-flow trajectory \(X_\varepsilon(t)=X(t)+\varepsilon \eta(t)\), where the perturbation \(\eta:[0,1]\to\mathbb{R}^{|E|}\) is smooth and satisfies the boundary conditions \(\eta(0)=0\), \(\eta(1)=0\)
and nonzero only inside $(0,1)$. The corresponding action is
\[
\mathcal{A}[X_\varepsilon]
=
\int_{0}^{\infty}\mathcal{L}(X_\varepsilon(t),\dot X_\varepsilon(t))\,dt.
\]
The action associated with this perturbed trajectory is
\[
\mathcal{A}[X_\varepsilon]
=
\frac12 \int_{0}^{\infty}
\|\dot X+\varepsilon\dot\eta\|_2^2
-
 (X+\varepsilon\eta)^\top \Delta_1 (X+\varepsilon\eta)
\,dt.
\]
Differentiating with respect to \(\varepsilon\) at \(\varepsilon=0\) gives
\[
\delta \mathcal{A}[X_\varepsilon]
=
\left.\frac{d}{d\varepsilon}\mathcal{A}[X_\varepsilon]\right|_{\varepsilon=0}
=
\int_{0}^{\infty}
\Big(
\dot X^\top \dot\eta
-
\eta^\top \Delta_1 X
\Big)\,dt,
\]
where we used the symmetry of \(\Delta_1\) to simplify the second term. Integration by parts gives
\[
\int_{0}^{\infty}\dot X^\top \dot\eta\,dt
=
\big[\dot X^\top \eta\big]_{0}^{\infty}
-
\int_{0}^{\infty}\ddot X^\top \eta\,dt.
\]
The boundary term vanishes because \(\eta(0)=0\) and \(\eta(t)=0\) for sufficiently large \(t\). Subsequently,
\bqn
\label{eq:variation}
\delta \mathcal{A}[X_\varepsilon]
=
-
\int_{0}^{\infty}
\eta(t)^\top
\big(
\ddot X(t)+\Delta_1 X(t)
\big)\,dt.
\eqn
The variation must be zero for \emph{every} choice of perturbation \(\eta(t)\). Thus,
$
\ddot X(t)+\Delta_1 X(t)=0,
$
which is the {\it Euler--Lagrange equation} \citep{arnold.2013} governing the evolution of  flow.
\end{proof}


\begin{proof}[Proof to Theorem 2]
From variation (\ref{eq:variation}) for Theorem \ref{thm:ELD},
\[
\delta \mathcal{A}[X]
=
-\int_0^\infty
\eta(t)^\top\big(\ddot X(t)+\Delta_1 X(t)\big)\,dt.
\]
For Rayleigh dissipation, we have
\[
- \int_0^\infty \langle F(t), \delta X(t)\rangle\,dt
=
\int_0^\infty \gamma\, \dot X(t)^\top \eta(t)\,dt,
\]
with \(\delta X(t)=\eta(t)\). Combining terms gives
\[
\int_0^\infty
\eta(t)^\top
\big(
\ddot X(t)+\gamma \dot X(t)+\Delta_1 X(t)
\big)\,dt
=
0.
\]
Since this equality should holds for all admissible perturbations \(\eta(t)\), the integral must vanish, yielding the damped Euler--Lagrange equation. The converse follows immediately by reversing the above steps.
\end{proof}


\begin{proof}[Proof to Theorem 3]

For any edge flow $X\in\mathbb{R}^{|E|}$, we have
\begin{align*}
\langle X,\Delta_1 X\rangle
&=\big\langle X,\mathbf{B}_1^\top\mathbf{B}_1 X\big\rangle
  +\big\langle X,\mathbf{B}_2\mathbf{B}_2^\top X\big\rangle\\
&=\big\langle \mathbf{B}_1 X,\mathbf{B}_1 X\big\rangle
  +\big\langle \mathbf{B}_2^\top X,\mathbf{B}_2^\top X\big\rangle\\
&=\|\mathbf{B}_1 X\|_2^2+\|\mathbf{B}_2^\top X\|_2^2.
\end{align*}

If $X\in\ker(\Delta_1)$, then $\Delta_1 X=0$ and hence
$
0
=
\langle X,\Delta_1 X\rangle
=
\|\mathbf{B}_1 X\|_2^2+\|\mathbf{B}_2^\top X\|_2^2.
$
Since both terms are nonnegative, each must vanish individually. Therefore
$\mathbf{B}_1 X=0$ and $\mathbf{B}_2^\top X=0$, implying
$X\in\ker(\mathbf{B}_1)\cap\ker(\mathbf{B}_2^\top)$. Conversely, if $X\in\ker(\mathbf{B}_1)\cap\ker(\mathbf{B}_2^\top)$, then
$\mathbf{B}_1 X=0$ and $\mathbf{B}_2^\top X=0$, which yields
$
\Delta_1 X
=
\mathbf{B}_1^\top(\mathbf{B}_1 X)
+
\mathbf{B}_2(\mathbf{B}_2^\top X)
=
0.
$
Thus $X\in\ker(\Delta_1)$. Subsequently, $\mathbf{B}_1 X = 0$ and $\mathbf{B}_2^\top X = 0$.
The condition $\mathbf{B}_1 X = 0$ means that, at every vertex,
the signed sum of incident edge flows vanishes. Thus $X$ has no
net inflow or outflow at any vertex, and therefore forms a closed
circulation, i.e., a $1$-cycle in simplicial homology \citep{hatcher.2002}.
The additional condition $\mathbf{B}_2^\top X = 0$ removes components
that can be written as boundaries of $2$-simplices (filled-in triangles).
Consequently, $X$ represents circulation along cycles that are not
boundaries of higher-dimensional faces. If $X \neq 0$, at least one such cyclic component is nontrivial.
\end{proof}

%

\begin{proof}[Proof to Theorem 4]
By definition, the total variance of the edge flow is
\[
\operatorname{Var}(X)
=
\mathbb{E}\|X-\bar{X}\|_2^2
=
\operatorname{tr}(\Sigma_X).
\]
Under the isotropic assumption $\Sigma_X=\sigma^2 I$, this becomes
$
\operatorname{Var}(X)
=
\operatorname{tr}(\sigma^2 I)
=
|E|\,\sigma^2.
$
The harmonic flow is obtained by orthogonal projection,
$
X_H = \mathcal{P}_H X,
$
where $\mathcal{P}_H$ is symmetric and idempotent, i.e.,
$\mathcal{P}_H^\top=\mathcal{P}_H$ and $\mathcal{P}_H^2=\mathcal{P}_H$.
The covariance of $X_H$ is therefore
\[
\Sigma_H
=
\mathbb{E}\!\left[(X_H-\mathbb{E}X_H)(X_H-\mathbb{E}X_H)^\top\right]
=
\mathcal{P}_H \Sigma_X \mathcal{P}_H.
\]
Substituting $\Sigma_X=\sigma^2 I$ yields
$
\Sigma_H
=
\sigma^2 \mathcal{P}_H.$ 
The total variance of the harmonic flow is then
\[
\operatorname{Var}(X_H)
=
\operatorname{tr}(\Sigma_H)
=
\sigma^2 \operatorname{tr}(\mathcal{P}_H).
\]
Since $\mathcal{P}_H$ is an orthogonal projector, its trace equals its rank,
which is the dimension of the harmonic subspace:
\[
\operatorname{tr}(\mathcal{P}_H)
=
\operatorname{rank}(\mathcal{P}_H)
=
\dim \ker(\Delta_1)
=
\beta_1.
\]
Hence,
\(
\operatorname{Var}(X_H)
=
\sigma^2 \beta_1
\)
and obtain the result. 
\end{proof}


\begin{proof}[Proof to Theorem 5]
Let the spectral decomposition of $\Sigma_X$ be
\(
\Sigma_X
=
\sum_{i=1}^{|E|} \sigma_i u_i u_i^\top, 
\sigma_1 \ge \sigma_2 \ge \cdots \ge \sigma_{|E|} \ge 0,
\)
where $\{u_i\}_{i=1}^{|E|}$ forms an orthonormal basis of eigenvectors.
The total variance of $X$ is
\(
\operatorname{Var}(X)
=
\mathbb{E}\|X-\bar{X}\|_2^2
=
\operatorname{tr}(\Sigma_X)
=
\sum_{i=1}^{|E|} \sigma_i.
\)
The covariance of the harmonic flow $X_H=\mathcal{P}_H X$ is
\[
\Sigma_H
=
\mathbb{E}\!\left[\mathcal{P}_H (X-\bar X)(X-\bar X)^\top \mathcal{P}_H\right]
=
\mathcal{P}_H \Sigma_X \mathcal{P}_H.
\]
Therefore, the total variance of $X_H$ is
\(
\operatorname{Var}(X_H)
=
\operatorname{tr}(\Sigma_H)
=
\operatorname{tr}(\Sigma_X \mathcal{P}_H),
\)
using the fact that $\mathcal{P}_H$ is  symmetric and idempotent, i.e.,
$\mathcal{P}_H^2=\mathcal{P}_H$ and
$\mathcal{P}_H^\top=\mathcal{P}_H$.

Let $\{\psi_j\}_{j=1}^{\beta_1}$ be an orthonormal basis of
$\ker(\Delta_1)$. Since $\mathcal{P}_H$ is the orthogonal projector onto
$\ker(\Delta_1)$, it admits the spectral representation
\(
\mathcal{P}_H
=
\sum_{j=1}^{\beta_1} \psi_j \psi_j^\top.
\)
Thus,
\[
\operatorname{Var}(X_H)
=
\operatorname{tr}\! \Big(
\Sigma_X \sum_{j=1}^{\beta_1} \psi_j \psi_j^\top
\Big)
=
\sum_{j=1}^{\beta_1} \psi_j^\top \Sigma_X \psi_j,
\]
which is a Rayleigh sum of $\Sigma_X$ over the
$\beta_1$-dimensional harmonic subspace \citep{horn.1985}. By Ky Fan’s inequality \citep{fan.1949}, we have
\[
\sum_{i=|E|-\beta_1+1}^{|E|} \sigma_i
\;\le\;
\sum_{j=1}^{\beta_1} \psi_j^\top \Sigma_X \psi_j
\;\le\;
\sum_{i=1}^{\beta_1} \sigma_i
\]
and obtain the result. 
\end{proof}

\end{document}